\documentclass[letterpaper, 10 pt, conference]{ieeeconf}  

\IEEEoverridecommandlockouts                              

\usepackage{graphicx} 
\usepackage{amsmath} 
\usepackage{amssymb}  
\usepackage{gensymb}
\usepackage{bm}
\usepackage{multirow}
\usepackage{subcaption}
\usepackage{xcolor}
\usepackage{hyperref}

\usepackage{amsthm}

\usepackage{bbm}
\usepackage{dsfont}

\usepackage{algorithm}
\usepackage{algpseudocode}
\makeatletter
\let\OldStatex\Statex
\renewcommand{\Statex}[1][0]{%
  \setlength\@tempdima{\algorithmicindent}%
  \OldStatex\hskip\dimexpr#1\@tempdima\relax}
\makeatother

\algnewcommand\algorithmicinput{\textbf{Input:}}
\algnewcommand\Input{\item[\algorithmicinput]}
\algnewcommand\algorithmicoutput{\textbf{Output:}}
\algnewcommand\Output{\item[\algorithmicoutput]}

\newtheorem{thm}{Theorem}

\newtheorem{lem}[thm]{Lemma}

\newtheorem{problem}{Problem}

\makeatletter
\let\NAT@parse\undefined
\makeatother

\DeclareCaptionLabelSeparator{periodspace}{. }
\usepackage[numbers, sort&compress]{natbib}

\title{\LARGE \bf Spectral Measurement Sparsification for Pose-Graph SLAM}
\author{Kevin J. Doherty,$^{1}$ David M. Rosen,$^{2}$ and John J.
  Leonard$^1$\thanks{$^1$K. Doherty and J. Leonard are with the Massachusetts
    Institute of Technology (MIT), Cambridge, MA 02139. $^{2}$D. Rosen is with
    Northeastern University, Boston, MA 02115.}}

\usepackage{accents}  
\newcommand{\ubar}[1]{\underaccent{\bar}{#1}}

\newcommand{\R}{\mathbb{R}}

\def \PSD{\mathbb{S}}

\def \transpose{^\mathsf{T}}

\DeclareMathOperator{\Graph}{graph}

\DeclareMathOperator{\SO}{SO}

\DeclareMathOperator{\SE}{SE}

\DeclareMathOperator*{\argmax}{argmax}

\DeclareMathOperator{\Gaussian}{\mathcal{N}}  
\DeclareMathOperator{\Langevin}{Langevin}  

\def \Graph {\mathcal{G}}  
\def \Nodes {\mathcal{V}}  
\def \Edges {\mathcal{E}}  
\newcommand{\directed}[1]{\vec{#1}}  
\def \dEdges{\directed{\Edges}}  

\def \edge{\lbrace i,j \rbrace}  
\def \dedge{(i,j)}  

\def \incEdges{\delta}

\def\Lap{L}  

\def \pose{x}
\def \tran{t}
\def \rot{R}

\newcommand{\true}[1]{\ubar{#1}}
\newcommand{\noisy}[1]{\tilde{#1}}
\newcommand{\est}[1]{\hat{#1}}

\def \tpose{\true{\pose}}
\def \ttran{\true{\tran}}
\def \trot{\true{\rot}}

\def \optsym{*}

\def \npose{\noisy{\pose}}
\def \ntran{\noisy{\tran}}
\def \nrot{\noisy{\rot}}

\def \transym{\tau}  

\def \LapRotW{\Lap}  
\def \LapRotWO{\Lap^{o}}
\def \LapRotWC{\Lap^{c}}

\def \Sorbdist{d_{\mathcal{S}}}

\def \projsym{\Pi}
\newcommand{\rounded}[1]{\projsym(#1)}

\def \selection{\omega}

\def \EO{\Edges^{o}}
\def \EC{\Edges^{c}}
\def \Eopt{\Edges^{\optsym}}

\def \objectiveF{F}

\def \dualF{\objectiveF_{D}}

\def \optValAC{p^{\optsym}}

\def \City10K{\emph{City10K}}
\def \Intel{\emph{Intel}}
\def \AIS2Klinik{\emph{AIS2Klinik}}
\def \KittiTwo{\emph{KITTI 02}}
\def \KittiFive{\emph{KITTI 05}}

\def \MAC{MAC}

\begin{document}
\maketitle

\begin{abstract}
  Simultaneous localization and mapping (SLAM) is a critical capability in autonomous navigation, but in order to scale SLAM to the setting of ``lifelong'' SLAM, particularly under memory or computation constraints, a robot must be able to determine what information should be retained and what can safely be forgotten. In graph-based SLAM, the number of edges (measurements) in a pose graph determines both the memory requirements of storing a robot's observations and the computational expense of algorithms deployed for performing state estimation using those observations; both of which can grow unbounded during long-term navigation. To address this, we propose a \emph{spectral} approach for pose graph sparsification which maximizes the \emph{algebraic connectivity} of the sparsified measurement graphs, a key quantity which has been shown to control the estimation error of pose graph SLAM solutions. Our algorithm, MAC (for \emph{maximizing algebraic connectivity}), which is based on convex relaxation, is simple and computationally inexpensive, and admits formal \emph{post hoc} performance guarantees on the quality of the solutions it provides. In experiments on benchmark pose-graph SLAM datasets, we show that our approach quickly produces high-quality sparsification results which retain the connectivity of the graph and, in turn, the quality of corresponding SLAM solutions, as compared to a baseline approach which does not consider graph connectivity.
\end{abstract}

\section{Introduction}

The problem of simultaneous localization and mapping (SLAM), in which a robot
aims to jointly infer its pose and the location of environmental landmarks, is a
critical capability in autonomous navigation. However, as we aim to scale SLAM
algorithms to the setting of ``lifelong'' autonomy, particularly on compute- or
memory-limited platforms, a robot must be able to determine \emph{what
  information should be kept, and what can safely be forgotten}
\cite{rosen2021advances}. In particular, in the setting of graph-based SLAM and
rotation averaging, the number of edges in a measurement graph determines both
the memory required to store a robot's observations as well as the computation
time of algorithms employed for state estimation using this measurement graph.

While there has been substantial work on the topic of measurement pruning (or
\emph{sparsification}) in lifelong SLAM (e.g. \cite{CarlevarisBianco13iros,
  carlevaris2014conservative, Johannsson12rssw, kurz2021geometry,
  kretzschmar2012information}), most existing methods rely on heuristics for
sparsification whereby little can be said about the quality of the statistical
estimates obtained from the sparsified graph versus the original. Recent work on
performance guarantees in the setting of pose-graph SLAM and rotation averaging
identified the spectral properties---specifically the \emph{algebraic
  connectivity} (also known as the \emph{Fiedler value})---of the measurement
graphs encountered in these problems to be central objects of interest,
controlling not just the best possible expected performance (per earlier work on
Cram\'er-Rao bounds \cite{boumal2014cramer, chen2021cramer,
  khosoussi2014novel}), but also the \emph{worst-case} error of estimators
\cite{rosen2019se, doherty2022performance, doherty2022performanceSupplement}.
These observations suggest the algebraic connectivity as a natural measure of
graph quality for assessing SLAM graphs. This motivates our use of the algebraic
connectivity as an \emph{objective} in formulating the graph sparsification
problem.

Specifically, we propose a spectral approach to pose graph sparsification which
maximizes the algebraic connectivity of the measurement graph subject to a
constraint on the number of allowed edges.\footnote{Our method is related to,
  but should not be confused with, \emph{spectral sparsification}
  \cite{spielman2011spectral}. Similar to spectral sparsification, we aim to
  sparsify graphs in a way that preserves their spectral properties. However,
  our method differs in that we focus only on the algebraic connectivity,
  whereas traditionally spectral sparsification aims to preserve the entire
  graph Laplacian spectrum.} As we discuss, this corresponds to \emph{E-optimal}
design in the setting of pose-graph SLAM \cite{pukelsheim2006optimal}. This
specific problem turns out to be an instance of the \emph{maximum algebraic
  connectivity augmentation} problem, which is NP-Hard \cite{mosk2008maximum}.
To address this, we propose to solve a computationally tractable relaxation and
\emph{round} solutions obtained to the relaxed problem to approximate feasible
solutions of the original problem. Relaxations of this form have been considered
previously; in particular \citet{ghosh2006growing} developed a semidefinite
program relaxation to solve problems of the form we consider. However, these
techniques do not scale to the size of typical problems encountered in
graph-based SLAM. To this end, we propose a first-order optimization approach
that we show is practically fast for even quite large SLAM problems. Moreover,
we show that the \emph{dual} to our relaxation provides tractable, high-quality
bounds on the suboptimality of the solutions we provide \emph{with respect to
  the original problem}.

In summary, we present an approach for pose graph sparsification by maximizing
the \emph{algebraic connectivity} of the measurement graph, a key quantity which
has been shown to control the estimation error of pose-graph SLAM solutions. Our
method, based on convex relaxation, is \emph{simple} and \emph{computationally
  inexpensive}, and admits formal \emph{post hoc} performance guarantees on the
quality of the solutions it provides. In experiments on several benchmark
pose-graph SLAM datasets, we show that our approach quickly produces
high-quality sparsification results which retain the connectivity of the graph
and better preserve the quality of SLAM solutions compared to a baseline which
does not consider graph connectivity.

The remainder of this paper proceeds as follows: In Section
\ref{sec:related-work} we discuss recent results on the importance of algebraic
connectivity in the context of pose-graph SLAM, previous work on the maximum
algebraic connectivity augmentation problem, and prior work on network design
with applications in pose graph sparsification. In Section
\ref{sec:prob-formulation} we give relevant pose-graph SLAM background and a
formal description of the maximum algebraic connectivity augmentation problem.
In Section \ref{sec:approach}, we formulate a relaxation of the algebraic
connectivity maximization problem, present a first-order optimization approach
for solving the relaxation, and describe a simple rounding procedure for
obtaining approximate solutions to the pose graph sparsification problem from a
solution to the relaxed problem. In Section \ref{sec:experimental-results} we
demonstrate our approach on several common pose-graph SLAM benchmark datasets
and show that compared to a ``topology unaware'' baseline, our solutions provide
improved graph connectivity and improved accuracy of the maximum-likelihood
estimators employed to solve pose-graph SLAM using the sparsified graphs.

\section{Related Work}\label{sec:related-work}

\subsection{Importance of the algebraic connectivity in SLAM}\label{sec:alg-conn-slam}

The importance of algebraic connectivity in general has been observed since at
least 1973, with the seminal work of \citet{fiedler1973algebraic}. In robot
perception, the algebraic connectivity has appeared in the context of rotation
averaging \cite{boumal2014cramer}, linear SLAM problems and sensor network
localization \cite{khosoussi2014novel, khosoussi2019reliable}, and pose-graph
SLAM \cite{chen2021cramer} as a key quantity controlling estimation performance.
In particular, \citet{boumal2014cramer} observed that the inverse of the
algebraic connectivity bounds (up to constants) the Cram\'er-Rao lower bound on
the expected mean squared error for rotation averaging. More recently, it
appeared as the key quantity controlling the \emph{worst-case} error of
estimators applied to measurement graphs in pose-graph SLAM and rotation
averaging \cite{rosen2019se, doherty2022performance,
  doherty2022performanceSupplement} (where \emph{larger} algebraic connectivity
is associated with (statistically) lower error).

\subsection{Maximizing the algebraic connectivity}

The problem of maximizing the algebraic connectivity subject to cardinality
constraints has been considered previously for a number of related applications.
\citet{ghosh2006growing} consider a semidefinite program relaxation of the same
objective we consider. Alternatively, \citet{nagarajan2018maximizing} considered
a mixed-integer approach to optimize this objective. While our overall approach
can make use of any solution to the relaxation we consider, neither of these
methods scales to the types of problems we are considering. To the best of our
knowledge, this is the first time an approach has been proposed for pose graph
sparsification which makes use of \emph{any} approach to solving a convex (or
concave) relaxation.

\subsection{Network design and pose graph sparsification}

The theory of optimal experimental design (TOED) \cite{pukelsheim2006optimal}
gives several optimality criteria applicable to network design. Specifically,
A-optimality, T-optimality, E-optimality, and D-optimality are common criteria,
each of which corresponds to optimizing a different property of the information
matrix describing the distribution of interest (in SLAM, this is typically the
joint distribution over robot and landmark states). Briefly, A-optimal designs
minimize the trace of the inverse of the information matrix, D-optimal designs
maximize the determinant of the information matrix, E-optimal designs maximize
the smallest eigenvalue of the information matrix, and T-optimal designs
maximize the trace of the information matrix. \citet{chen2021cramer} discuss the
connections between Cram\'er-Rao bounds for pose-graph SLAM and the A-optimality
and T-optimality criteria. Historically, Cram\'er-Rao bounds and the optimal
design metrics arising from them have been popular tools for network design and
active SLAM \cite{khosoussi2014novel, chen2021cramer}, including, e.g. in
application to the planning of underwater inspection routes \cite{Kim09oceans}.

\citet{khosoussi2019reliable} established many of the first results for optimal
graph sparsification (i.e. measurement subset selection) in the setting of SLAM.
The convex relaxation they consider is perhaps the closest existing work in the
literature to ours. However, in contrast to our approach, they consider the
D-optimality criterion, while the results discussed in Section
\ref{sec:alg-conn-slam} strongly suggest that the quantity of interest with
regard to estimation performance is the algebraic connectivity, and therefore
the E-optimality criterion.\footnote{Of course, by maximizing the smallest
  eigenvalue, the E-optimality criterion also selects for information matrices
  with larger determinant.} More practically, the E-optimality criterion is both
less computationally expensive to compute \emph{and} to optimize.

Several methods have been proposed to reduce the \emph{number of states} which
need to be estimated in a SLAM problem (e.g. \cite{CarlevarisBianco13iros,
  carlevaris2014conservative, Johannsson12rssw, Huang13ECMRb}), typically by
marginalizing out state variables. This procedure is usually followed by an edge
pruning operation to mitigate the unwanted increase in graph density. Previously
considered approaches rely on linearization of measurement models at a
particular state estimate in order to compute approximate marginals and perform
subsequent pruning. Consequently, little can be said concretely about the
quality of the statistical estimates obtained from the sparsified graph compared
to the original graph. In contrast, our approach does not require linearization,
and provides explicit performance guarantees on the graph algebraic connectivity
as compared to the globally optimal algebraic connectivity (which is itself
linked to both the \emph{best} and \emph{worst} case performance of estimators
applied to the SLAM problem).

\section{Problem Formulation}\label{sec:prob-formulation}

We consider graph sparsification in the setting of pose-graph SLAM. Pose-graph
SLAM is the problem of estimating $n$ unknown values $\pose_i, \ldots, \pose_n
\in \SE(d)$ given a subset of measurements of their pairwise relative transforms
$\npose_{ij}$. This problem admits a natural graphical representation $\Graph
\triangleq (\Nodes, \dEdges)$ where nodes $\Nodes$ correspond to latent
variables $\pose_i \in \SE(d)$ and edges $(i,j) \in \dEdges$ correspond to noisy
measurements of relative poses $\npose_{ij}$. We adopt the following generative
model for rotation and translation measurements: For each edge $\dedge \in
\dEdges$:
\begin{subequations}
  \begin{equation}
    \nrot_{ij} = \trot_{ij}\rot_{ij}^\epsilon, \quad \rot_{ij}^{\epsilon} \sim \Langevin(I_d, \kappa_{ij})
  \end{equation}
  \begin{equation}
    \ntran_{ij} = \ttran_{ij} + \tran_{ij}^\epsilon, \quad \tran_{ij}^{\epsilon} \sim \Gaussian(0, \transym_{ij}^{-1}I_d),
  \end{equation}
  \label{eq:gen-model-pgo}
\end{subequations}%
where $\tpose_{ij} = (\ttran_{ij}, \trot_{ij})$ is the true value of
$\pose_{ij}$. That is, $\tpose_{ij} = \tpose_i^{-1}\tpose_j$ given the true
values of poses $\pose_i$ and $\pose_j$. Under this noise model, the typical
nonlinear least-squares formulation of maximum-likelihood estimation (MLE) for
$\SE(d)$ synchronization is written as follows:
\begin{problem}[MLE for $\SE(d)$ synchronization]
  \label{se-mle}
  \begin{equation}
    \min_{\substack{{\tran_i \in \R^d} \\ {\rot_i \in \SO(d)}}} \sum_{(i,j) \in \dEdges} \kappa_{ij} \| \rot_j - \rot_i\nrot_{ij}\|^2_F + \transym_{ij}\|t_j - t_i - R_i\ntran_{ij}\|_2^2.
  \end{equation}%
\end{problem}
Problem \ref{se-mle} also directly captures the problem of rotation averaging
under the Langevin noise model simply by taking all $\tau_{ij} = 0$
\cite{dellaert2020shonan}.

In prior work \cite{doherty2022performance, doherty2022performanceSupplement,
  rosen2019se}, we showed that the smallest (nonzero) eigenvalue of the
\emph{rotational weight Laplacian} $\LapRotW$ controls the worst-case error of
solutions to Problem \ref{se-mle}; this is the algebraic connectivity (or
Fiedler value) of the graph with nodes in correspondence with robot poses
$\pose_i$ and edge weights equal to each $\kappa_{ij}$. The corresponding
eigenvector attaining this value is called the \emph{Fiedler vector}. The
rotational weight Laplacian $\LapRotW$ is the $n \times n$ matrix with
$i,j$-entries:
\begin{equation}
  \LapRotW_{ij} = \begin{cases} \sum_{e \in \incEdges(i)} \kappa_e,& i = j, \\
    -\kappa_{ij},& \edge \in \Edges, \\
    0,& \edge \notin \Edges. \end{cases}
  \label{eq:laprotw}
\end{equation}
where $\incEdges(i)$ denotes the set of edges \emph{incident to} node $i$. The
Laplacian of a graph has several well-known properties that we will use here.
The Laplacian $\LapRotW$ of a graph can be written as a sum of the Laplacians of
the subgraphs induced by each of its edges. A Laplacian is always
positive-semidefinite, and the ``all ones'' vector $\mathds{1}$ of length $n$
is always in its kernel. Finally, a graph has positive algebraic connectivity
$\lambda_2(\LapRotW) > 0$ if and only if it is connected.\footnote{More
  specifically, the number of zero eigenvalues of a Laplacian is equal to the
  number of connected components of its corresponding graph.}

It will be convenient to partition the edges as $\Edges = \EO \cup \EC,\ \EO
\cap \EC = \emptyset$ into a \emph{fixed} set of edges $\EO$ and a set of $m$
\emph{candidate} edges $\EC$, and where $\LapRotWO$ and $\LapRotWC$ are the
Laplacians of the subgraphs induced by $\EO$ and $\EC$. For our purposes, the
subgraph induced by $\EO$ on $\Nodes$ will typically be constructed from
sequential odometric measurements (therefore, $|\EO| = n - 1$), but this is not
a requirement of our general approach.\footnote{In particular, to apply our
  approach we should select $\LapRotWO$ and $K$ to guarantee that the feasible
  set for Problem \ref{prob:max-aug-alg-conn} contains at least one tree. Then,
  it is clear that the optimization in Problem \ref{prob:max-aug-alg-conn} will
  always return a connected graph, since $\lambda_2(\LapRotW(\selection)) > 0$
  if and only if the corresponding graph is connected. Note that this condition
  is always easy to arrange: for example, we can start with $\LapRotWO$ a tree,
  as we do here, or (even more simply) take $\LapRotWO$ to be the zero matrix
  and simply take $K \geq n - 1$.} It will be helpful in the subsequent
presentation to ``overload'' the definition of $\LapRotW$. Specifically, let
$\LapRotW : \R^m \rightarrow \PSD_{n \times n}$ be the affine map constructing
the total graph Laplacian from a weighted combination of edges in $\EC$:
\begin{equation}
  \LapRotW(\selection) \triangleq \LapRotWO + \sum_{k = 1}^m \selection_{k}\LapRotWC_{k},
\end{equation}
where $\LapRotWC_k$ is the Laplacian of the subgraph induced by edge $e_k =
\{i_k, j_k\}$ of $\EC$. Our goal in this work will be to identify a subset of
$\Eopt \subseteq \EC$ of fixed size $|\Eopt| = K$ (equivalently, the edge
selection $\selection$), which maximizes the algebraic connectivity
$\lambda_2(\LapRotW(\selection))$. This corresponds to the following
optimization problem:
\begin{problem}[Algebraic connectivity maximization]
  \label{prob:max-aug-alg-conn}
  \begin{equation}
    \begin{gathered}
      \optValAC = \max_{\selection \in \{0,1\}^m} \lambda_{2}(\LapRotW(\selection)) \label{eq:connectivity-rot-only} \\
      |\selection | = K.
    \end{gathered}
  \end{equation}
\end{problem}

\section{Approach}\label{sec:approach}

Problem \ref{prob:max-aug-alg-conn} is a variant of the \emph{maximum algebraic
  connectivity augmentation problem}, which is NP-Hard \cite{mosk2008maximum}.
The difficulty of Problem \ref{prob:max-aug-alg-conn} stems, in particular, from
the integrality constraint on the elements of $\selection$. Consequently, our
general approach will be to solve a simpler problem obtained by relaxing the
integrality constraints of Problem \ref{prob:max-aug-alg-conn}, and, if
necessary, \emph{rounding} the solution to the relaxed problem to a solution in
the feasible set of Problem \ref{prob:max-aug-alg-conn}. In particular, we
consider the following \emph{Boolean relaxation} of Problem
\ref{prob:max-aug-alg-conn}:

\begin{problem}[Boolean Relaxation of Problem \ref{prob:max-aug-alg-conn}]
  \begin{equation}\label{eq:relaxation}
    \begin{gathered}
      \max_{\selection \in [0,1]^m} \lambda_2(\LapRotW(\selection)) \\
      \mathds{1}\transpose \selection = K.
    \end{gathered}
  \end{equation}
  \label{prob:relaxation}
\end{problem}

Relaxing the integrality constraints of Problem \ref{prob:max-aug-alg-conn}
dramatically alters the difficulty of the problem. In particular, we know (cf.
\cite{ghosh2006growing}):
\begin{lem}\label{lem:concavity}
  The function $\objectiveF(\selection) = \lambda_2(\LapRotW(\selection))$ is
  \emph{concave} on the set $\selection \in [0,1]^m,
  \mathds{1}\transpose\selection = K$.
\end{lem}
Consequently, solving Problem \ref{prob:relaxation}, then, amounts to maximizing
a concave function over a convex set; this is in fact a convex optimization
problem (one can see this by simply considering minimization of the objective
$-\objectiveF(\selection)$) and hence \emph{globally solvable} (see, e.g.
\cite{boyd2004convex, bertsekas2016nonlinear}). Since a solution to Problem
\ref{prob:relaxation} need not be feasible for the original problem, we then
\emph{round} solutions to the relaxed problem to their nearest correspondents in
the feasible set of Problem \ref{prob:max-aug-alg-conn}.

\subsection{Solving the relaxation}

\begin{small}
  \begin{algorithm}[t]
    \caption{MAC Algorithm \label{alg:mac}}
    \begin{algorithmic}[1]
      \Input An initial iterate $\selection$
      \Output An approximate solution to Problem \ref{prob:max-aug-alg-conn}
      \Function{MAC}{$\selection$}
      \State $\selection \leftarrow $\textsc{FrankWolfeAC}$(\selection)$
      \Comment{Solve Problem \ref{prob:relaxation}}
      \State \Return $\rounded{\selection}$
      \Comment{Round solution; eq. \eqref{eq:rounding}}
      \EndFunction
    \end{algorithmic}
  \end{algorithm}
\end{small}

\begin{small}
  \begin{algorithm}[t]
    \caption{Frank-Wolfe Method for Problem \ref{prob:relaxation} \label{alg:frank-wolfe}}
    \begin{algorithmic}[1]
      \Input An initial feasible iterate $\selection$
      \Output An approximate solution to Problem \ref{prob:relaxation}
      \Function{FrankWolfeAC}{$\selection$}
      \For {$t = 0, \dotsc, T$}
      \State Compute a Fiedler vector $y^*$ of $\LapRotW(\omega)$
      \State $\nabla \objectiveF(\omega)_k \leftarrow
      {y^*}\transpose\LapRotWC_k{y^*}, k = 1, \ldots, m$
      \Comment{Eq. \eqref{eq:supergradient}}
      \State $s_t \leftarrow \argmax_s s\transpose \nabla \objectiveF(\selection) $
      \Comment{Prob. \ref{prob:dir-subproblem}; eq. \eqref{eq:dir-subproblem-opt}}
      \State $\alpha \leftarrow 2 / (2 + t)$
      \Comment{Compute step size}
      \State $\selection \leftarrow \selection + \alpha \left( s_t - \selection \right)$
      \EndFor
      \State \Return $\selection$
      \EndFunction
    \end{algorithmic}
  \end{algorithm}
\end{small}

There are several methods which could, in principle, be used to solve the
relaxation in Problem \ref{prob:relaxation}. For example, Ghosh and Boyd
\cite{ghosh2006growing} consider solving an equivalent semidefinite program.
This approach has the advantage of fast convergence (in terms of the number of
iterations required to compute an optimal solution), but can nonetheless be slow
for the large problem instances ($m > 1000$) typically encountered in the SLAM
setting. Instead, our algorithm for \emph{maximizing algebraic connectivity}
(\MAC{}), summarized in Algorithm \ref{alg:mac}, employs an inexpensive
subgradient (more precisely, \emph{supergradient}) approach to solve Problem
\ref{prob:relaxation}, then \emph{rounds} its solution to the nearest element of
the feasible set for Problem
\ref{prob:max-aug-alg-conn}.\footnote{Supergradients are simply the concave
  analogue of subgradients; i.e., the tangent hyperplane formed by any
  supergradient of a concave function $\objectiveF$ must lie \emph{above}
  $\objectiveF$.}

In particular, \MAC{} uses the \emph{Frank-Wolfe method} (also known as the
\emph{conditional gradient method}), a classical approach for solving convex
optimization problems of the form in Problem \ref{prob:relaxation}
\cite{bertsekas2016nonlinear}. At each iteration, the Frank-Wolfe method
requires (1) linearizing the objective $\objectiveF$ at a particular
$\selection$, (2) maximizing the linearized objective over the (convex) feasible
set, and (3) taking a step in the direction of the solution to the linearized
problem. The remainder of this section gives a detailed exposition of our
adaptation of the Frank-Wolfe method to the problem of algebraic connectivity
maximization, which is summarized in Algorithm \ref{alg:frank-wolfe}.

The Frank-Wolfe method is particularly advantageous in this setting since the
feasible set for Problem \ref{prob:relaxation} is the intersection of the
hypercube with the linear subspace determined by $\mathds{1}\transpose
\selection = K$ (a linear equality constraint). Consequently this problem
amounts to solving a linear program, which can be done easily (and in fact, as
we show, admits a simple closed-form solution). In particular, the
\emph{direction-finding subproblem} for the Frank-Wolfe method is the following
linear program:
\begin{problem}[Direction-finding subproblem]\label{prob:dir-subproblem}
  Fix an iterate $\selection \in [0,1]^m,\ \mathds{1}\transpose \selection = K$. The
  direction-finding subproblem is to find the point $s$ solving the following
  linear program:
  \begin{equation}
    \begin{gathered}
      \max_{s \in [0,1]^m} s\transpose \nabla \objectiveF(\selection), \\
      \mathds{1}\transpose  s = K.
    \end{gathered}
  \end{equation}
\end{problem}
In order to compute the linearized objective in Problem
\ref{prob:dir-subproblem} we require a supergradient of the original objective
function (in the usual case where $\objectiveF$ is differentiable at
$\selection$, this is simply the gradient of $\objectiveF$). It turns out, we
can always recover a supergradient of $\objectiveF$ at a particular $\selection$
in terms of a Fiedler vector of $\LapRotW(\selection)$. Specifically, we have
the following theorem (which we prove in Appendix A):
\begin{thm}[Supergradients of $\objectiveF(\selection)$]\label{thm:supergradient}
  Let $y^*(\selection)$ be any eigenvector of $\LapRotW(\selection)$
  corresponding to $\lambda_2(\LapRotW(\selection))$. Then:
\begin{equation}\label{eq:supergradient}
  \begin{gathered}
    \nabla \objectiveF(\selection) = \left[ \frac{\partial \objectiveF}{\partial \selection_1}, \ldots, \frac{\partial \objectiveF}{\partial \selection_m} \right]\transpose, \\
    \frac{\partial F}{\partial \selection_k} = y^*(\selection)\transpose\LapRotWC_ky^*(\selection),
  \end{gathered}
\end{equation}
is a \emph{supergradient} of $\objectiveF$ at $\selection$.
\end{thm}
Therefore, supergradient computation can be performed by simply recovering an
eigenvector of $\LapRotW(\selection)$ corresponding to
$\lambda_2(\LapRotW(\selection))$.

Problem \ref{prob:dir-subproblem} is a linear program, for which several
solution techniques exist \cite{bertsekas2016nonlinear}. However, in our case,
Problem \ref{prob:dir-subproblem} admits a simple, \emph{closed-form} solution
$s^*$ attaining its optimal value (which we prove in Appendix B):
\begin{thm}[A closed-form solution to Problem \ref{prob:dir-subproblem}]\label{thm:dir-subproblem}
  Let $\mathcal{S}^*,\ |\mathcal{S}^*| = K$ be the set containing the indices of
  the $K$ \emph{largest} elements of $\nabla \objectiveF(\selection)$, breaking
  ties arbitrarily where necessary. The vector $s^* \in \R^n$ with element $k$
  given by:
  \begin{equation}\label{eq:dir-subproblem-opt}
    s^*_k = \begin{cases} 1, &  k \in \mathcal{S}^*, \\
      0, &  \text{otherwise,} \end{cases}
  \end{equation}
  is an optimizer for Problem \ref{prob:dir-subproblem}.
\end{thm}
In this work, we use a simple decaying step size $\alpha$ to update $\selection$
in each iteration. While in principle, we could instead use a line search method
\cite[Sec. 2.2]{bertsekas2016nonlinear}), this would potentially require many
evaluations of $\objectiveF(\selection)$ within each iteration. Since every
evaluation of $\objectiveF(\selection)$ requires an eigenvalue computation, this
can become a computational burden for large problems.

In the event that the optimal solution to the relaxed problem is integral, we
ensure that we have \emph{also} obtained an optimal solution to the original
problem. However, this need not be the case in general. In the (typical) case
where integrality does not hold, we \emph{project} the solution to the relaxed
problem onto the original constraint set. In this case, an integral solution
$\rounded{\selection}$ can be obtained by rounding the largest $K$ components of
$s$ to 1, and setting all other components to zero:
\begin{equation}\label{eq:rounding}
  \rounded{\selection}_k \triangleq \begin{cases} 1, & \text{if $\selection_k$ is in the largest $K$ elements of $\selection$}, \\
  0, & \text{otherwise}. \end{cases}
\end{equation}

In general, the Frank-Wolfe algorithm offers \emph{sublinear} (i.e.
$\mathcal{O}(1/T)$ after $T$ iterations) convergence to the globally optimal
solution in the worst case \cite{dunn1978conditional}. However, in this context
it has several advantages over alternative approaches. First, we can bound the
sparsity of a solution after $T$ iterations. In particular, we know that the
solution after $T$ iterations has \emph{at most} $KT$ nonzero entries. Second,
the gradient computation requires only a single computation of the minimal $2$
dimensional eigenspace of an $n \times n$ matrix. This can be performed quickly
using a variety of methods (e.g. the preconditioned Lanczos method). Finally, as
we showed, the direction-finding subproblem in Problem \ref{prob:dir-subproblem}
admits a simple \emph{closed-form solution} (as opposed to a projected gradient
method which requires projection onto an $\ell_1$-ball). In consequence, despite
the fact that gradient-based methods may require many iterations to converge to
\emph{globally optimal} solutions, high-quality approximate solutions can be
computed fast at the scale necessary for SLAM problems. As we show in the
following section, our approach admits \emph{post hoc} suboptimality guarantees
even in the event that we terminate optimization prematurely (e.g. when a
\emph{fast} but potentially suboptimal solution is required). Critically, these
suboptimality guarantees ensure the quality of the solutions of our approach not
only with respect to the relaxation, but also with respect to the \emph{original
  problem}.

\subsection{Post-hoc Suboptimality Guarantees}

Algorithm \ref{alg:frank-wolfe} admits several \emph{post hoc} suboptimality
guarantees. Let $\optValAC$ be the optimal value of the original nonconvex
maximization in Problem \ref{prob:max-aug-alg-conn}. Since Problem
\ref{prob:relaxation} is a relaxation of Problem \ref{prob:max-aug-alg-conn}, in
the event that optimality attains for a vector $\selection^*$, we know:
\begin{equation}
  \objectiveF(\rounded{\selection^*}) \leq \optValAC \leq \objectiveF(\selection^*).
\end{equation}
Therefore, the suboptimality of a rounded solution $\rounded{\selection^*}$ is bounded
as follows:
\begin{equation}
  \optValAC - \objectiveF(\rounded{\selection^*}) \leq \objectiveF(\selection^*) - \objectiveF(\rounded{\selection^*}).
\end{equation}
Consequently, in the event that $\objectiveF(\selection^*) -
\objectiveF(\rounded{\selection^*}) = 0$, we know that $\rounded{\selection^*}$
\emph{must } correspond to an optimal solution to Problem
\ref{prob:max-aug-alg-conn}.

The above guarantees apply in the event that we obtain a \emph{maximizer}
$\selection^*$ of Problem \ref{prob:relaxation}. This would seem to pose an
issue if we aim to terminate optimization before we obtain a verifiable,
globally optimal solution to Problem \ref{prob:relaxation} (e.g. in the presence
of real-time constraints). Since these solutions are not necessarily globally
optimal in the relaxation, we do not know if their objective value is larger or
smaller than the optimal solution to Problem \ref{prob:max-aug-alg-conn}.
However, we can in fact obtain per-instance suboptimality guarantees of the same
kind for \emph{any} estimate $\est{\selection}$ through the \emph{dual} of our
relaxation (cf. \citet[Appendix D]{lacoste2013block}). Here, we give a
derivation of the dual upper bound which uses only the concavity of
$\objectiveF$.

Since $\objectiveF$ is concave, for any $x, y \in [0,1]^m,\
\mathds{1}\transpose x = \mathds{1}\transpose y = K$ we have:
\begin{equation}
  \objectiveF(y) \leq \objectiveF(x) + (y - x)\transpose \nabla \objectiveF(x).
\end{equation}
Consider then the following upper bound:
\begin{equation}\label{eq:dual-bound-deriv}
  \begin{aligned}
    \objectiveF(\selection^*) &\leq \objectiveF(\est{\selection}) + (\selection^* - \est{\selection})\transpose \nabla \objectiveF(\est{\selection}) \\
    &\leq \max_{s \in [0,1]^m, \mathds{1}\transpose s = K} \objectiveF(\est{\selection}) + (s - \est{\selection})\transpose \nabla \objectiveF(\est{\selection}) \\
    &= \objectiveF(\est{\selection}) - \est{\selection}\transpose \nabla \objectiveF(\est{\selection}) + \max_{s \in [0,1]^m, \mathds{1}\transpose s = K} s\transpose \nabla \objectiveF(\est{\selection}).
  \end{aligned}
\end{equation}
We observe that the solution to the optimization in the last line of
\eqref{eq:dual-bound-deriv} is \emph{exactly} the solution to the
direction-finding subproblem (Problem \ref{prob:dir-subproblem}). Letting
$\est{s}$ be a vector obtained as a solution to Problem
\ref{prob:dir-subproblem} at $\est{\selection}$, we obtain the following
\emph{dual} upper bound:
\begin{equation}\label{eq:dual-upper-bound}
  \dualF(\est{\omega}) \triangleq \objectiveF(\est{\selection}) + \nabla \objectiveF(\est{\selection})\transpose (\est{s} - \est{\selection}). \\
\end{equation}
Now, from \eqref{eq:dual-bound-deriv}, we have $\dualF(\selection) \geq
\objectiveF(\selection^*)$ for any $\selection$ in the feasible set. In turn, it
is straightforward to verify that the following chain of inequalities hold for
any estimator $\est{\omega}$ in the feasible set of the Boolean relaxation:
\begin{equation}
  \objectiveF(\rounded{\est{\selection}}) \leq \optValAC \leq \dualF(\est{\selection}),
\end{equation}
with the corresponding suboptimality guarantee:
\begin{equation}
  \optValAC - \objectiveF(\rounded{\est{\selection}}) \leq \dualF(\est{\selection}) - \objectiveF(\rounded{\est{\selection}}).
\end{equation}
Moreover, we can always recover a suboptimality bound on $\est{\selection}$ with
respect to the optimal value $\objectiveF(\selection^*)$ to relaxed problem as:
\begin{equation} \label{eq:duality-gap-as-bound}
 \objectiveF(\selection^*) - \objectiveF(\est{\selection}) \leq \dualF(\est{\selection}) - \objectiveF(\est{\selection})
\end{equation}
The expression appearing on the right-hand side of
\eqref{eq:duality-gap-as-bound} is the (Fenchel) \emph{duality gap}. Equation
\eqref{eq:duality-gap-as-bound} also motivates the use of the duality gap as a
stopping criterion for Algorithm \ref{alg:frank-wolfe}: if the gap is
sufficiently close to zero (e.g. to within a certain numerical tolerance), we
may conclude that we have reached an optimal solution $\selection^*$ to Problem
\ref{prob:relaxation}.

\begin{figure*}[t]
  \centering
  \begin{subfigure}{1.0\linewidth}
  \includegraphics[width=0.2\linewidth]{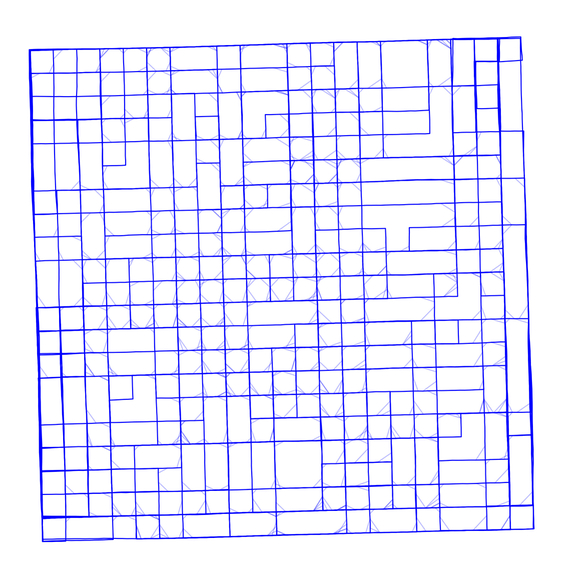}%
  \includegraphics[width=0.2\linewidth]{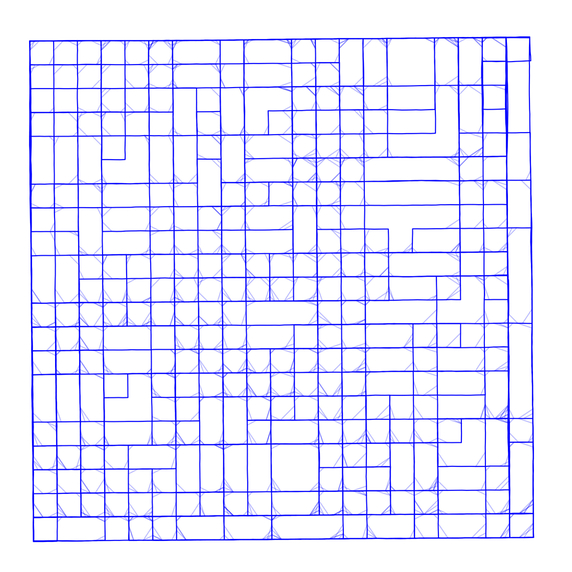}%
  \includegraphics[width=0.2\linewidth]{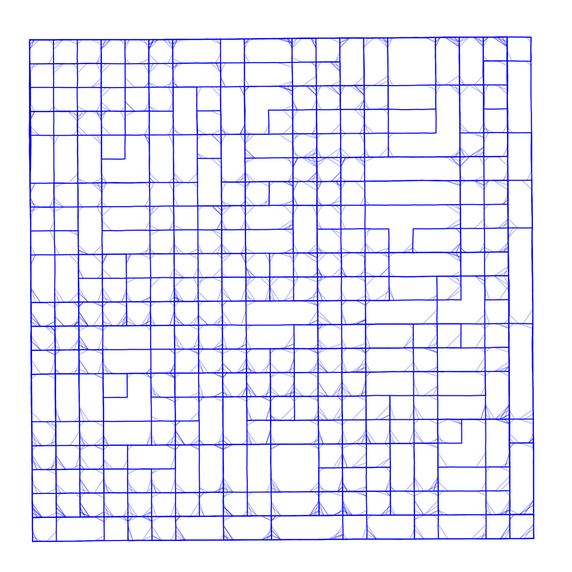}%
  \includegraphics[width=0.2\linewidth]{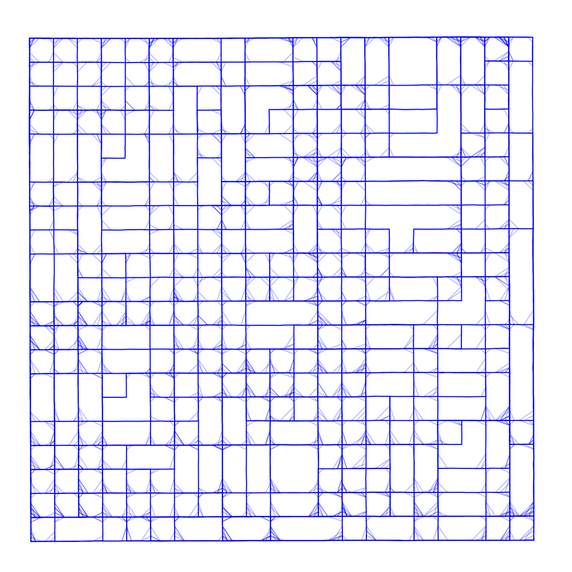}%
  \includegraphics[width=0.2\linewidth]{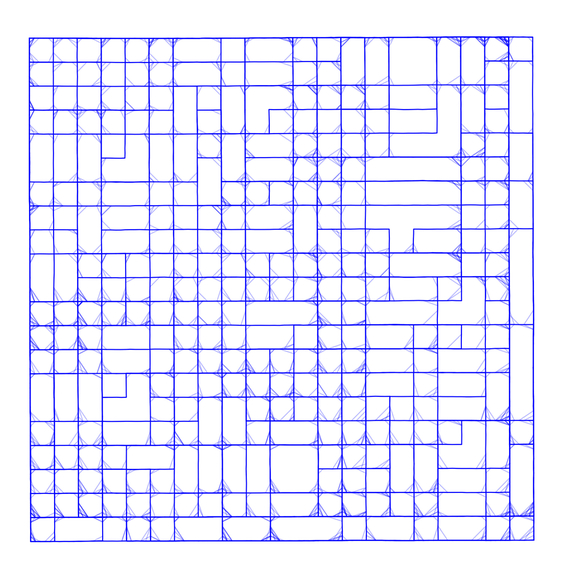}
  \caption{Our approach.}
  \end{subfigure}

  \begin{subfigure}{1.0\linewidth}
    \includegraphics[width=0.2\linewidth]{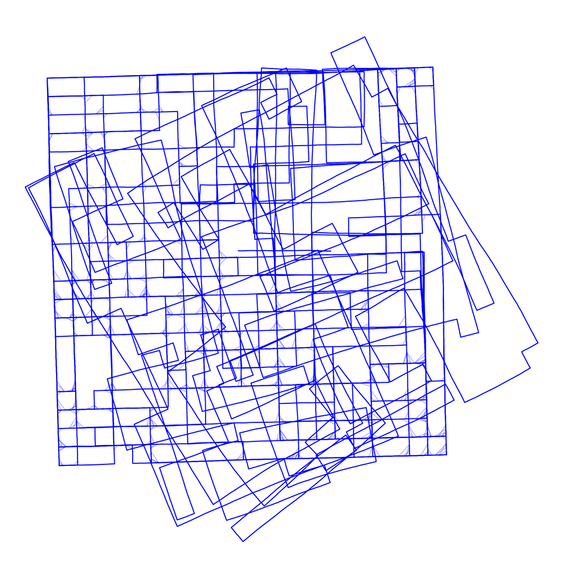}%
    \includegraphics[width=0.2\linewidth]{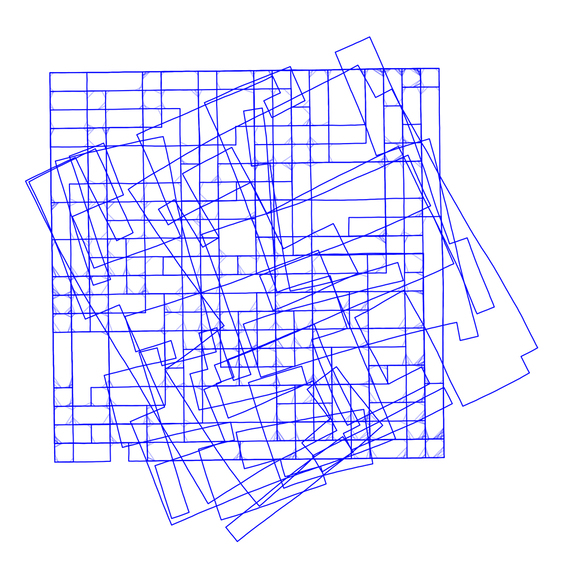}%
    \includegraphics[width=0.2\linewidth]{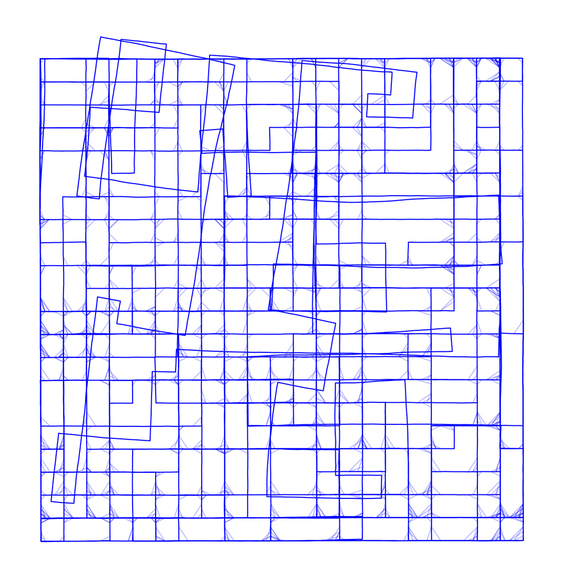}%
    \includegraphics[width=0.2\linewidth]{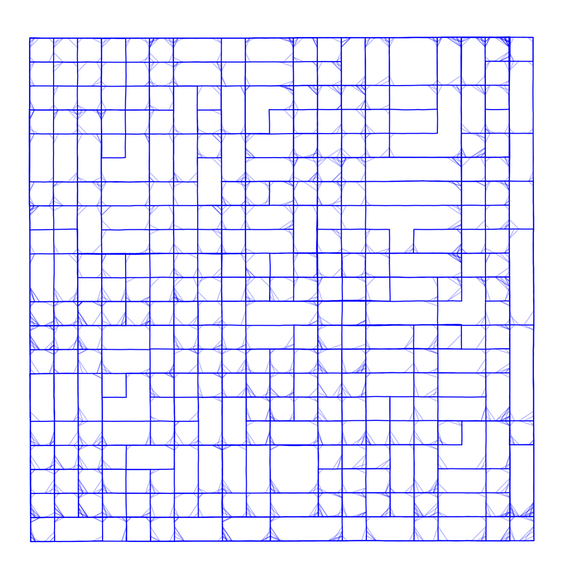}%
    \includegraphics[width=0.2\linewidth]{smaller_figs/ours_city_1.0.jpg}
    \caption{Na\"ive baseline.}
  \end{subfigure}
  \caption{\textbf{Qualitative results for pose-graph sparsification}. Pose-graph
      optimization results for the \City10K{} dataset with varying degrees of
      sparsity using (a) our method and (b) a na\"ive baseline which selects the
      most certain measurements. Left to right: 20\%, 40\%, 60\%, 80\%, and
      100\% of the candidate edges. \label{fig:qualitative}}
\end{figure*}

\section{Experimental Results}\label{sec:experimental-results}

\begin{figure*}
  \centering
  \begin{subfigure}{1.0\linewidth}
    \centering
    \includegraphics[width=0.25\linewidth]{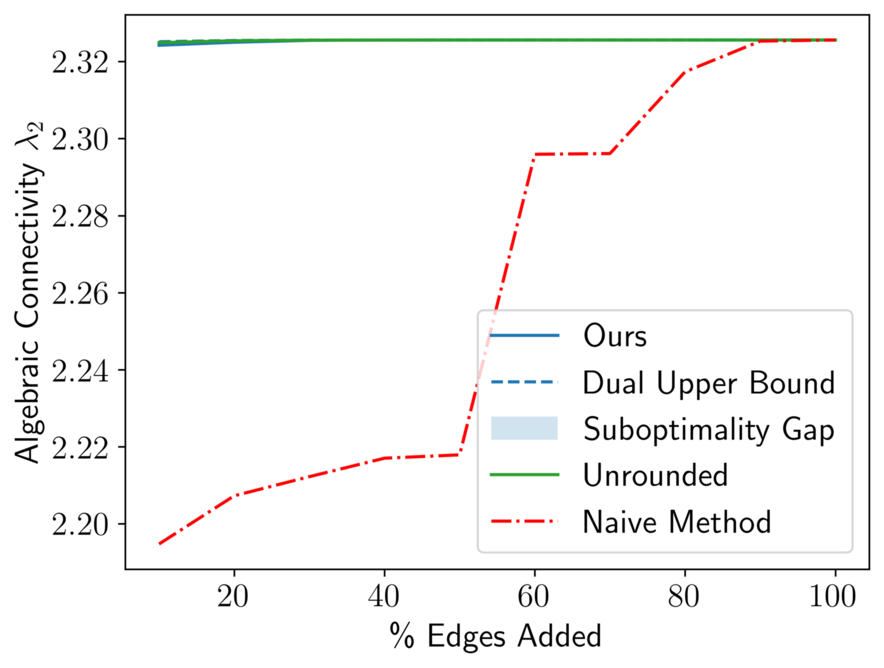}%
    \includegraphics[width=0.25\linewidth]{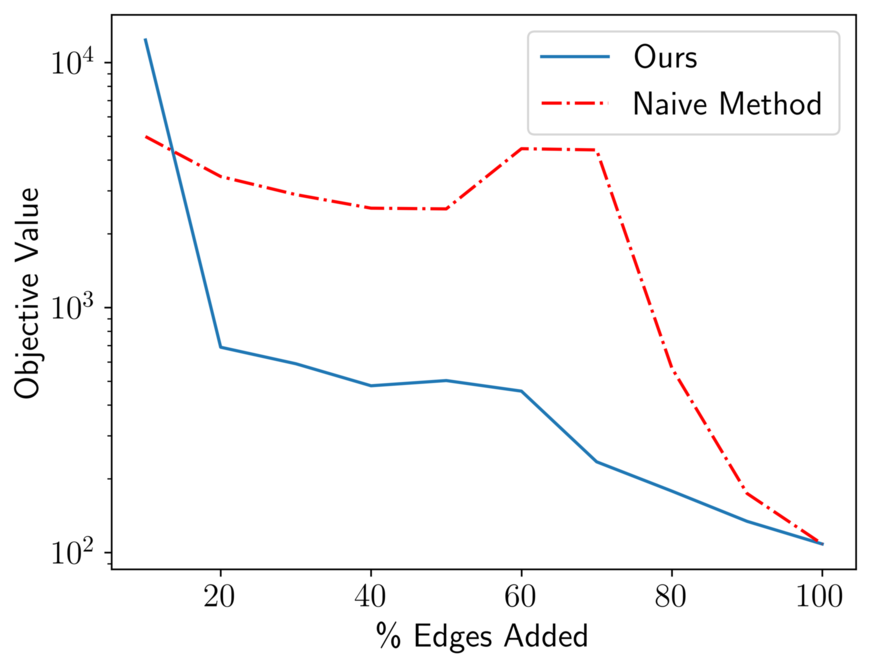}%
    \includegraphics[width=0.25\linewidth]{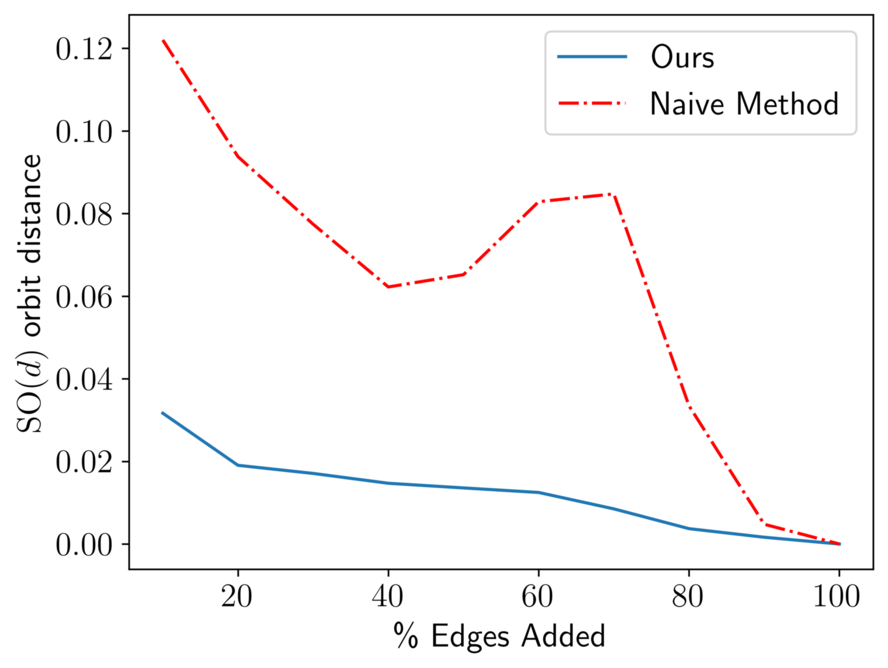}%
    \includegraphics[width=0.25\linewidth]{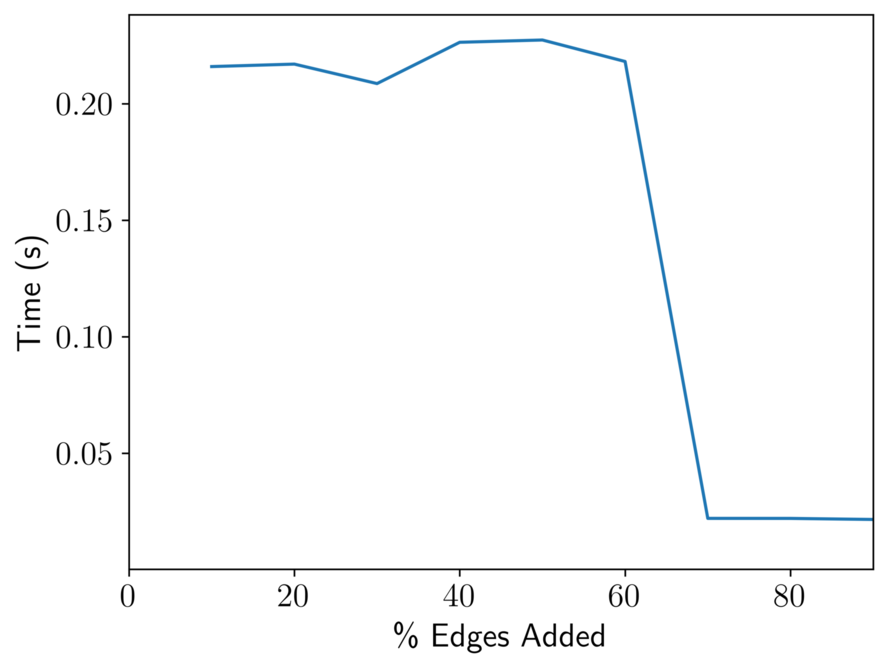}%
    \caption{\KittiTwo{}\\ \label{fig:quantitative:kitti02}}
  \end{subfigure}
  \begin{subfigure}{1.0\linewidth}
    \centering
    \includegraphics[width=0.25\linewidth]{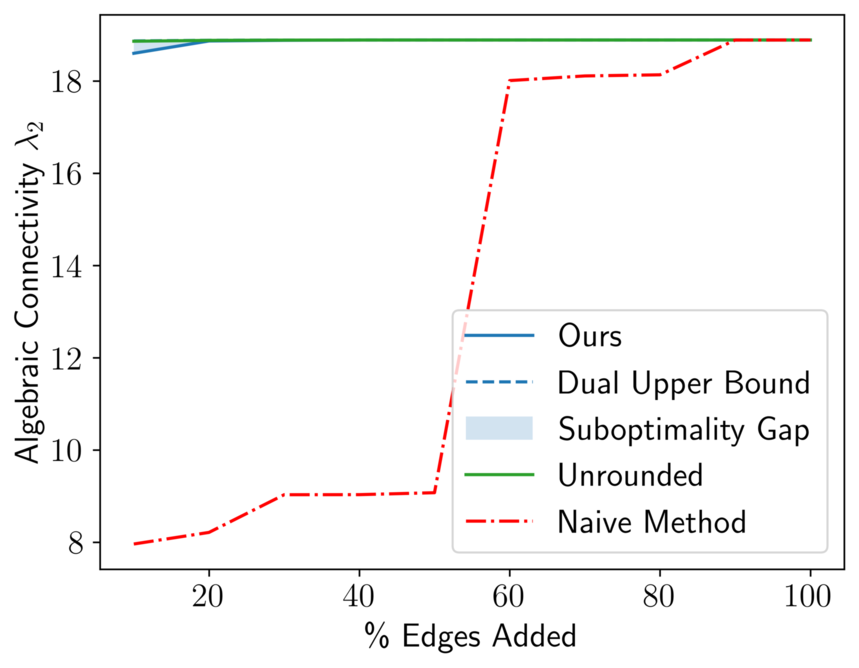}%
    \includegraphics[width=0.25\linewidth]{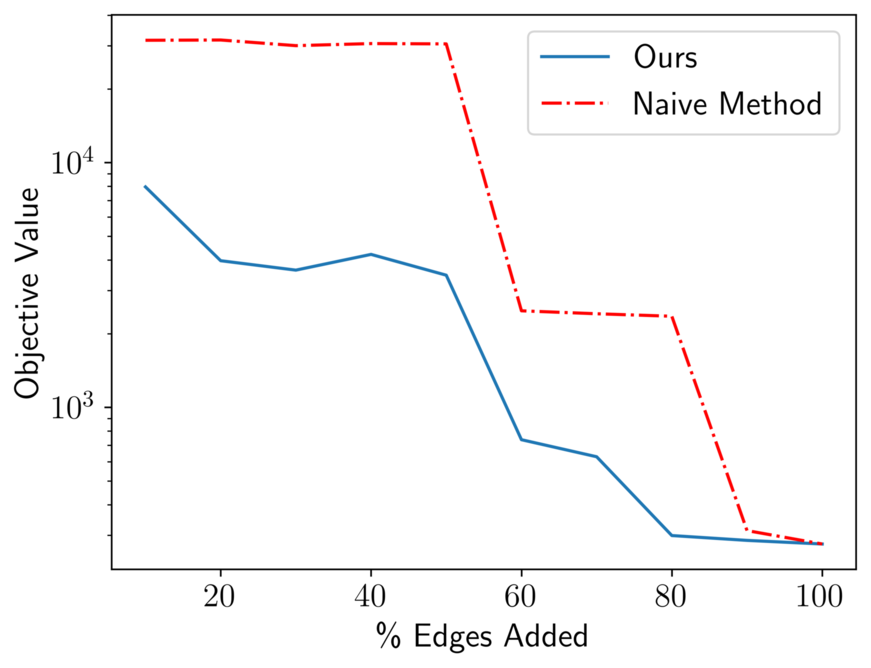}%
    \includegraphics[width=0.25\linewidth]{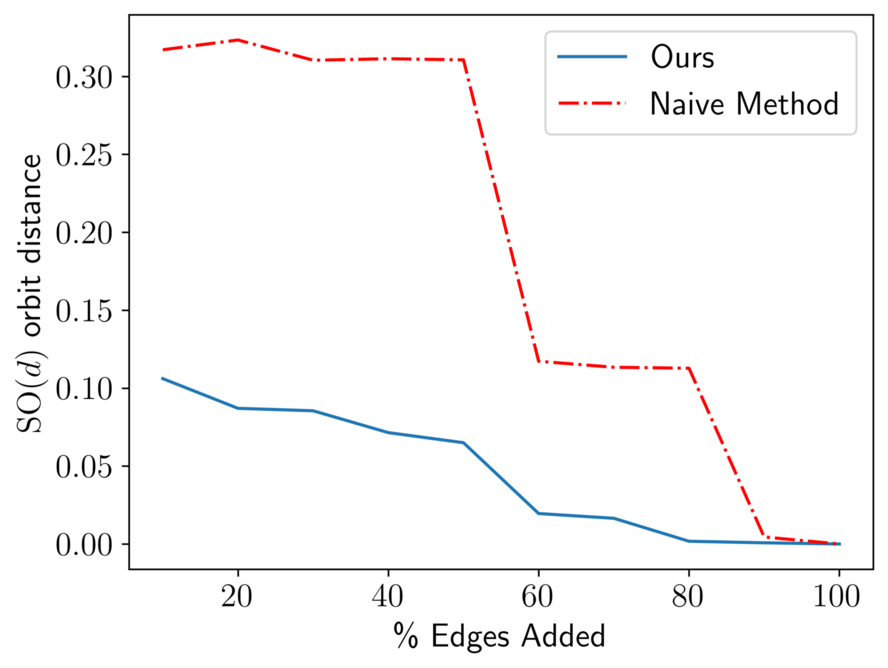}%
    \includegraphics[width=0.25\linewidth]{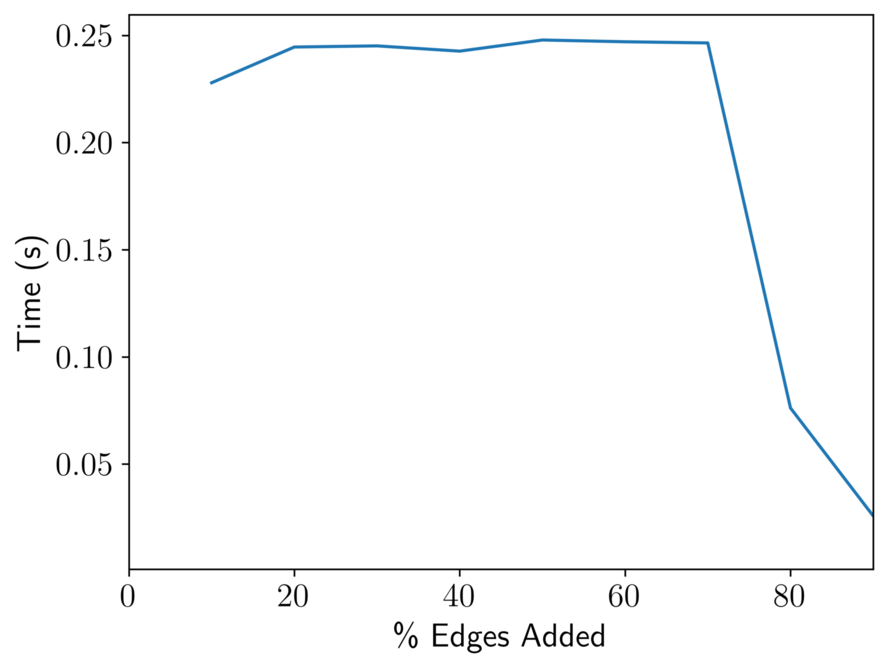}%
    \caption{\KittiFive{}\\ \label{fig:quantitative:kitti05}}
  \end{subfigure}
  \begin{subfigure}{1.0\linewidth}
    \centering
    \includegraphics[width=0.25\linewidth]{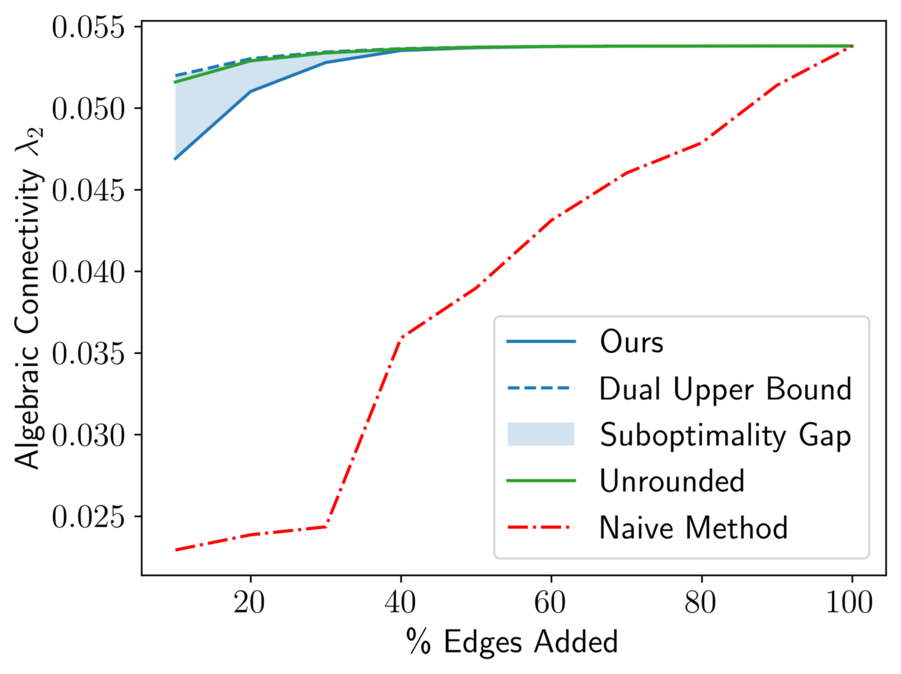}%
    \includegraphics[width=0.25\linewidth]{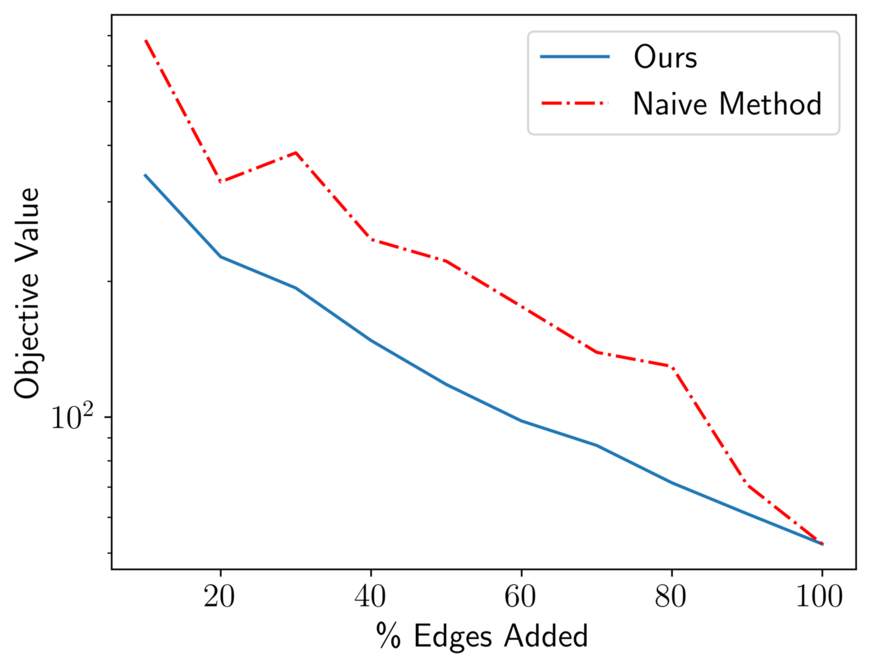}%
    \includegraphics[width=0.25\linewidth]{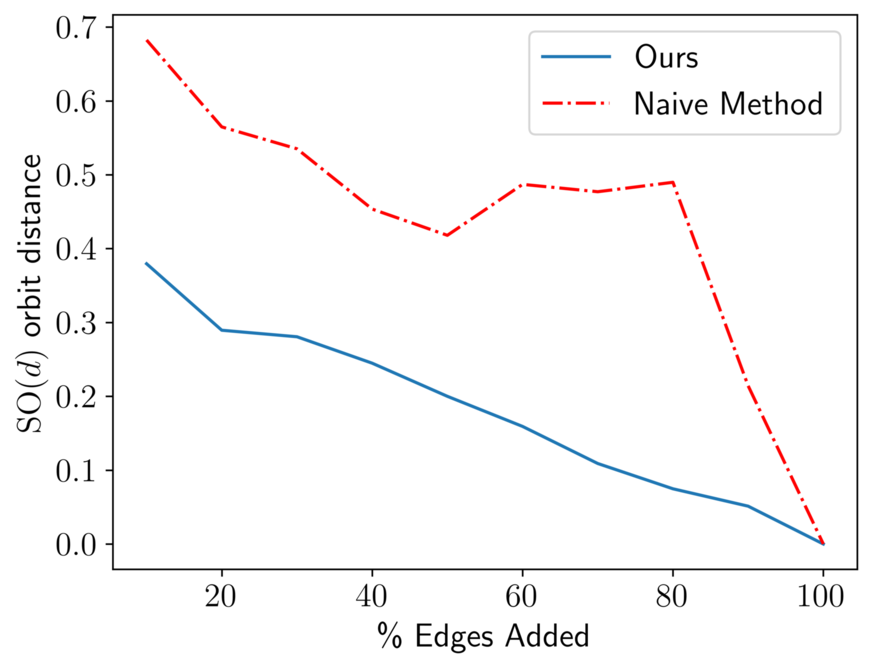}%
    \includegraphics[width=0.25\linewidth]{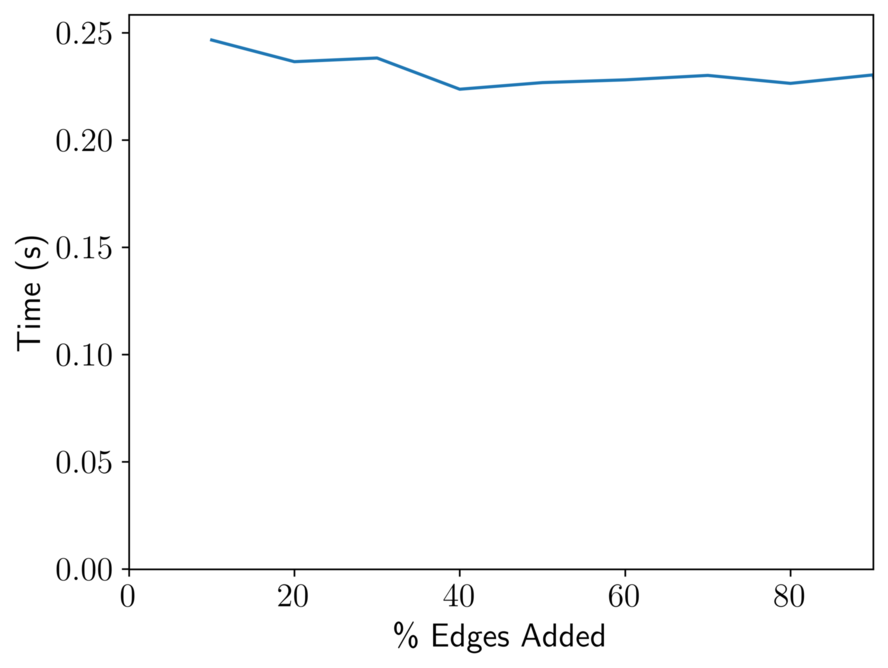}%
    \caption{\Intel{}\\ \label{fig:quantitative:intel}}
  \end{subfigure}
  \begin{subfigure}{1.0\linewidth}
    \centering
    \includegraphics[width=0.25\linewidth]{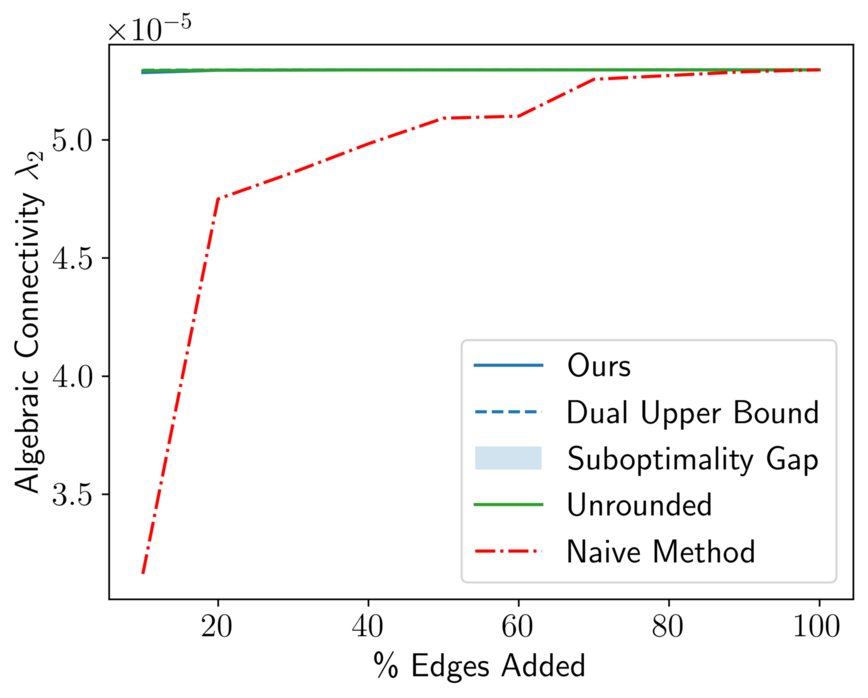}%
    \includegraphics[width=0.25\linewidth]{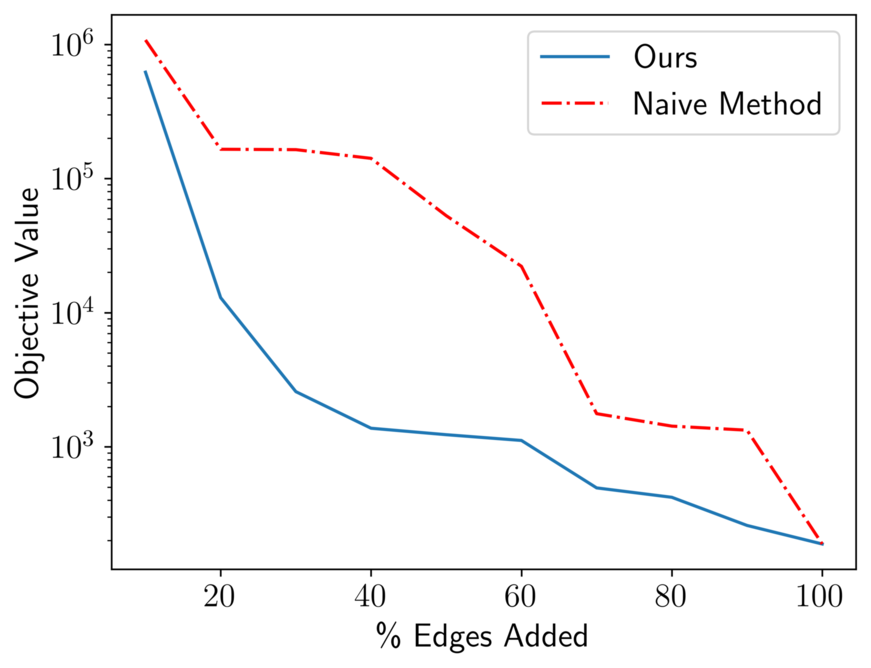}%
    \includegraphics[width=0.25\linewidth]{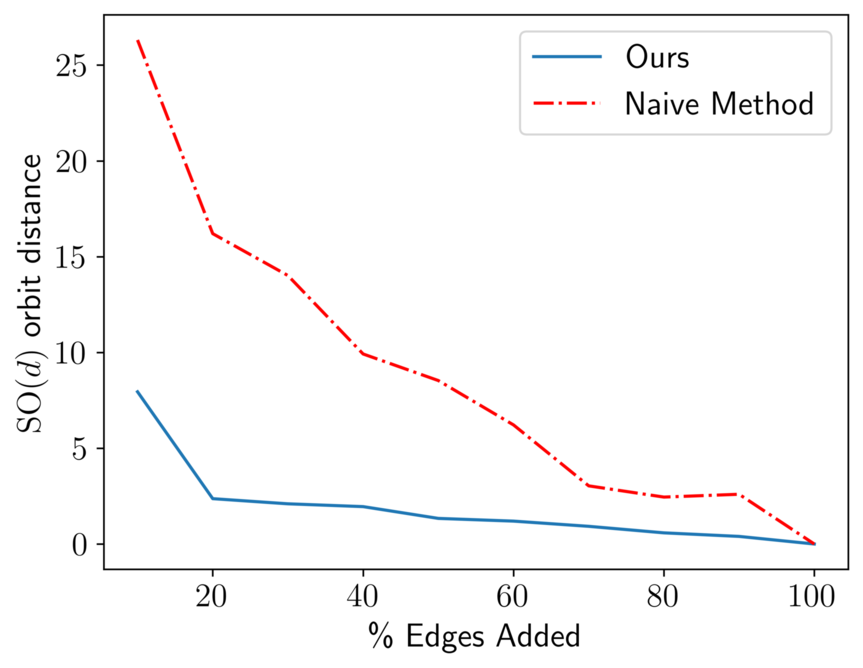}%
    \includegraphics[width=0.25\linewidth]{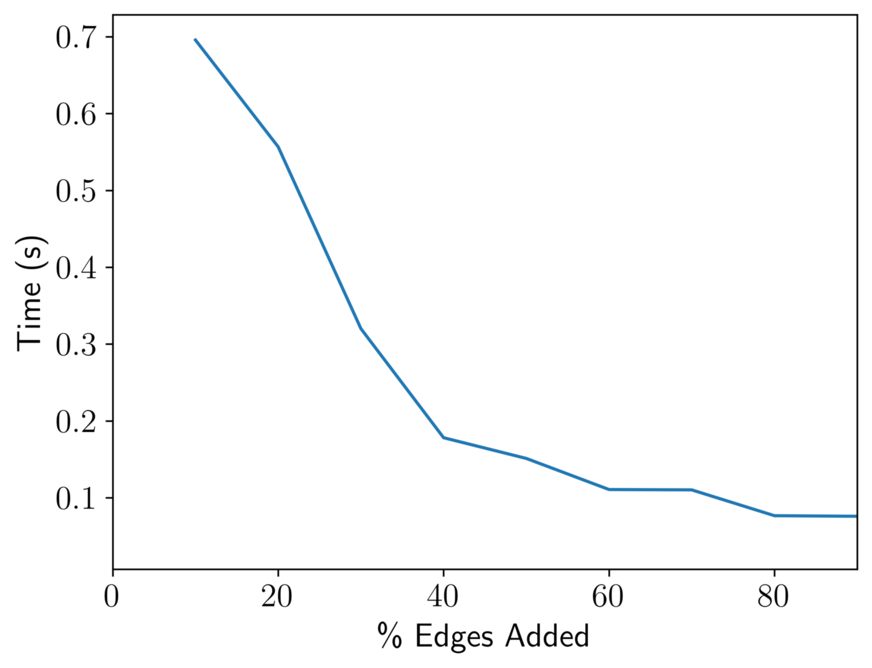}%
    \caption{\AIS2Klinik{}\\ \label{fig:quantitative:ais2klinik}}
  \end{subfigure}
  \begin{subfigure}{1.0\linewidth}
    \centering
    \includegraphics[width=0.25\linewidth]{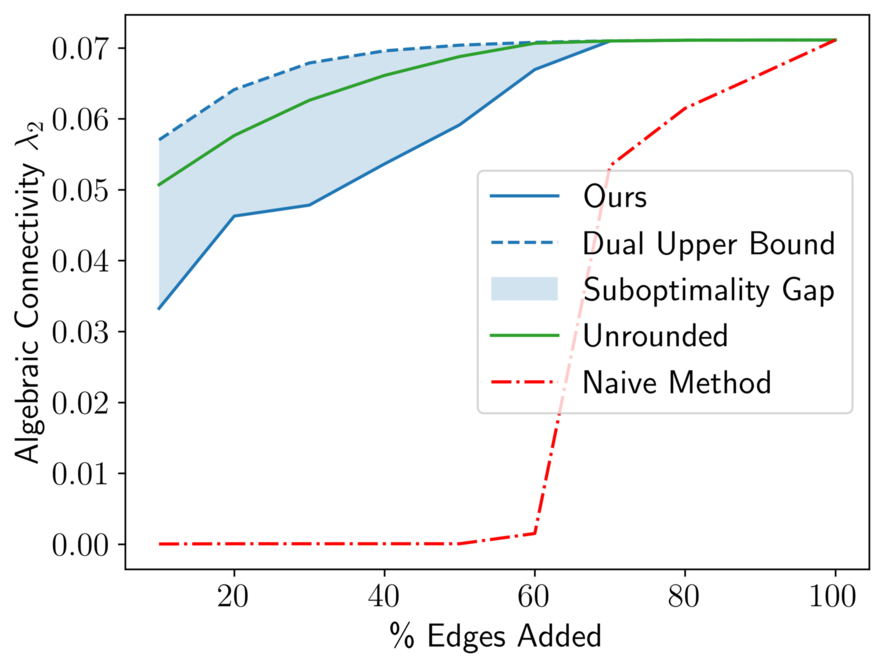}%
    \includegraphics[width=0.25\linewidth]{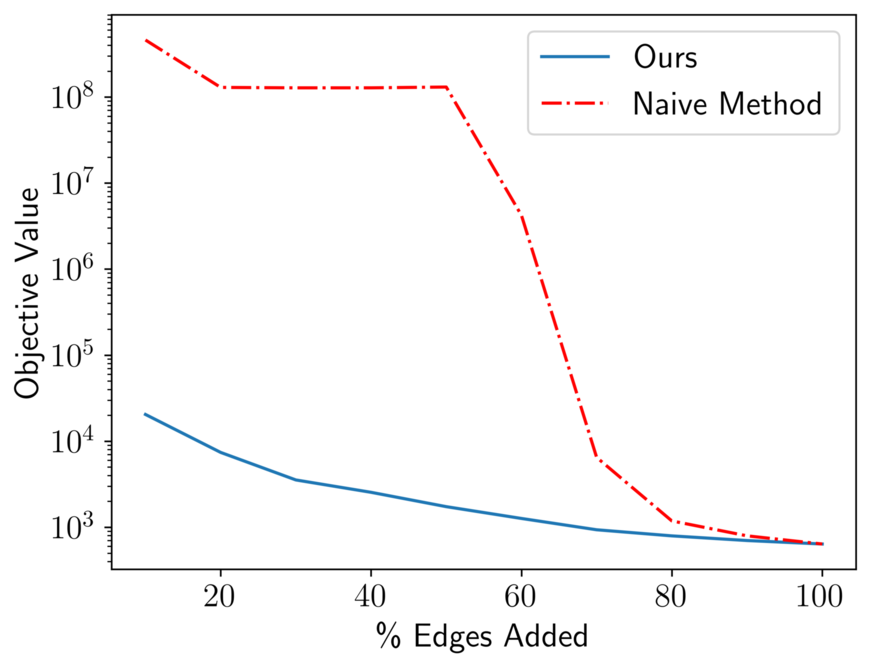}%
    \includegraphics[width=0.25\linewidth]{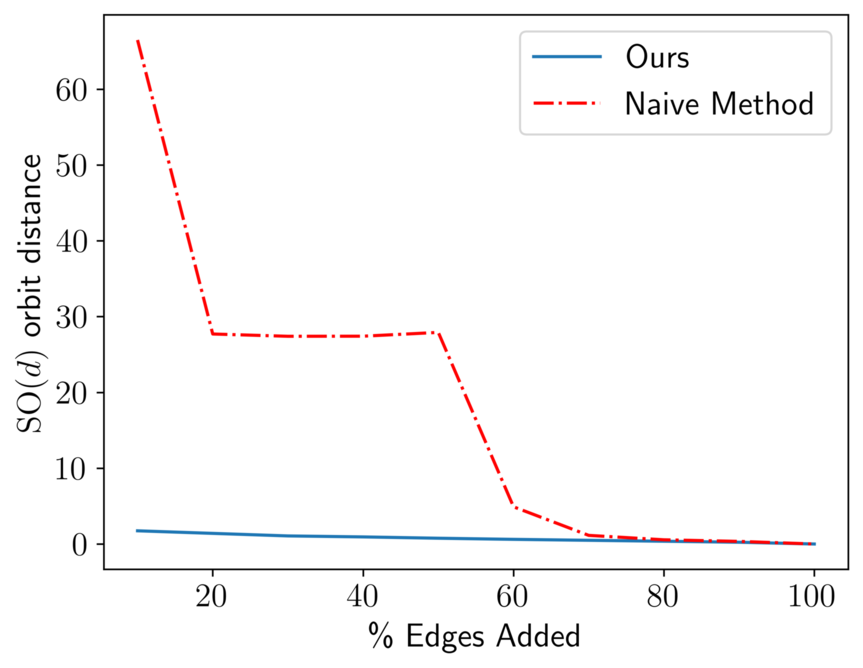}%
    \includegraphics[width=0.25\linewidth]{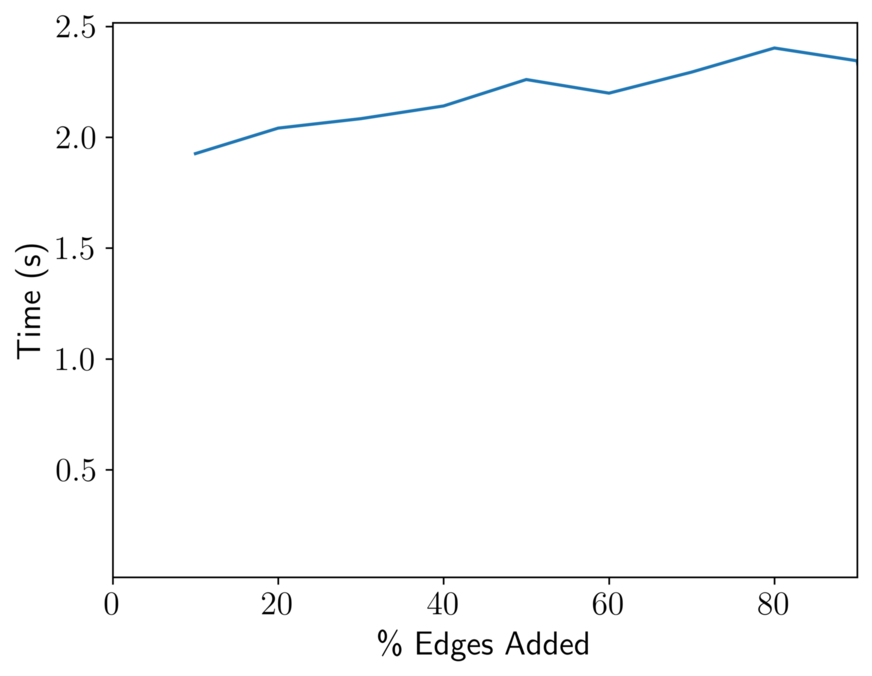}%
    \caption{\City10K{}\\ \label{fig:quantitative:city10k}}
  \end{subfigure}
  \caption{\textbf{Quantiative results for pose-graph sparsification}.
    Pose-graph optimization results for (a) the \KittiTwo{} dataset, (b) the
    \KittiFive{} dataset, (c) the \Intel{} dataset, (d) the \AIS2Klinik{}
    dataset, and (e) the \City10K{} dataset with varying degrees of sparsity (as
    percent of candidate edges added). Left to right: The algebraic connectivity
    of the graphs obtained by our method versus the na\"ive baseline (larger is
    better), the objective value of the maximum-likelihood estimator for each
    sparsified problem under the \emph{original} objective, i.e. with all edges
    retained (smaller is better; note the log-scale), the $\SO(d)$ orbit
    distance between a maximum-likelihood estimator computing using the
    sparsified graph and a corresponding maximum-likelihood estimator computed
    for the graph containing \emph{all} of the candidate edges (smaller is
    better), and the computation time for our approach. For 100\% loop closures,
    both algorithms return immediately, so no computation time is
    reported. \label{fig:quantitative}}
\end{figure*}

\begin{table}[t]
  \centering
  \setlength{\tabcolsep}{4pt}
  \begin{tabular}{c | c | c }
    \hline
    Dataset & No. of Nodes & No. of Candidate (Loop Closure) Edges  \\
    \hline
    \KittiTwo{} & 4661 & 43 \\
    \KittiFive{} & 2761 & 66 \\
    \Intel{} & 1728 & 785 \\
    \AIS2Klinik{} & 15115 & 1614 \\
    \City10K{} & 10000 & 10688 \\
    \hline   
  \end{tabular}
  \caption{Summary of the datasets used in our experiments.\label{table:datasets}}
\end{table}

We implemented the \MAC{} algorithm in Python and all computational experiments
were performed on a 2.4 GHz Intel i9-9980HK CPU. For computation of the Fiedler
value and the corresponding vector, we use TRACEMIN-Fiedler
\cite{manguoglu2010tracemin, sameh1982trace}. In all experiments, we run \MAC{}
for a maximum of 20 iterations, or when the duality gap in equation
\eqref{eq:duality-gap-as-bound} reaches a tolerance of $10^{-8}$.

We evaluated our approach using several benchmark pose-graph SLAM datasets. For
each dataset, we use odometry edges (between successive poses) to form the base
graph and loop closure edges as candidate edges. We consider selection of $10\%,
20\%, \ldots, 100\%$ of the candidate loop closure edges in the sparsification
problem. We present results on five datasets in this document (summarized in
Table \ref{table:datasets}). In particular, we consider here the \Intel{}
dataset, the \City10K{} dataset, KITTI dataset sequences 02, and 05, and the
\AIS2Klinik{} dataset. The \Intel{} dataset and the \AIS2Klinik{} dataset are
both obtained from real data, while the \City10K{} dataset is synthetic. The
\City10K{} dataset, however, contains far more candidate edges, and therefore
serves as a reasonable ``stress test'' for the computation time of our approach.
We compare our approach to a na\"ive heuristic method which does not consider
graph topology. Specifically, the na\"ive method selects the edges with the most
confident rotation measurements (i.e. the set of $K$ edges $\edge$ with the
largest $\kappa_{ij}$). This simple heuristic approach serves two purposes:
First, it provides a baseline, topology-agnostic approach to demonstrate the
impact of considering graph connectivity in a sparsification procedure; second,
we use this method to provide a \emph{sparse} initial estimate to our algorithm.
For each method, we compare the graph connectivity (as measured by the Fiedler
value) as well as the quality of maximum-likelihood estimators for pose-graph
optimization (i.e. solutions to Problem \ref{se-mle}) under the edge sets
selected by each method. We use SE-Sync \cite{rosen2019se} to compute the
globally optimal estimate of robot poses in each case.\footnote{In all of our
  experiments, SE-Sync returned \emph{certifiably-optimal} solutions to Problem
  \ref{se-mle}.}

Figure \ref{fig:qualitative} gives a qualitative comparison of the results from
our approach as compared with the baseline on the \City10K{} dataset across a
range of candidate loop closures allowed. We observe that even retaining $60\%$
of the candidate edges, the quality of the results provided by the baseline
method degrade significantly compared to those of the full set of loop closures.
In contrast, our sparsification approach leads to high-quality estimates even
with a \emph{significant} reduction in the number of edges.

For a quantitative comparison of each method, we report three performance
measures: (1) the algebraic connectivity $\lambda_2(\LapRotW(\selection))$ of
the graphs determined by each edge selection $\selection$, (2) the ``full''
objective value from Problem \ref{se-mle} (i.e. keeping 100\% of the edges)
attained by globally optimal solutions to the \emph{sparsified} problems, and
(3) the $\SO(d)$-orbit distance between the rotational states of a
maximum-likelihood estimator for the sparsified problem and those of a
maximum-likelihood estimator for the original (full) objective. The
$\SO(d)$-orbit distance between two rotational state estimates is defined as:
\begin{equation} \label{eq:sod-orbit-dist}
  \begin{gathered}
  \Sorbdist(X, Y) \triangleq \min_{G \in \SO(d)} \| X - GY \|_F, \\ X,Y \in \SO(d)^n,
  \end{gathered}
\end{equation}
which can be computed in closed form by means of a singular value decomposition
(see \citet[Theorem 5]{rosen2019se}). The ``full'' objective value attained by
solutions to the sparsified problems serves as one indicator of ``how close''
solutions to the sparsified problem are to the MLE for the ``full'' problem. If
the ``full'' objective value attained by the MLE for a sparsified graph is close
to that of the MLE computed using $100\%$ of the candidate edges, the MLE for
the sparsified graph is likely also a high-quality solution under the full
objective. The $\SO(d)$-orbit distance quantifies the actual deviation (up to
global symmetry) between the estimated rotational states in each solution. Since
the translational states are recovered analytically (per \cite{rosen2019se}),
this serves as a useful measure, independent of the global scale of the
translational states, of the degradation in solution quality from the ``full''
MLE as we sparsify the graph.

Figure \ref{fig:quantitative} summarizes our quantitative results on each
benchmark dataset. Our approach consistently achieves better connected graphs
(as measured by the algebraic connectivity). In most cases, a maximum of 20
iterations was enough to achieve solutions to the relaxation with algebraic
connectivity very close to the dual upper bound (and therefore nearly globally
optimal). Moreover, maximum-likelihood estimators for Problem \ref{se-mle}
computed using the sparsified measurement graphs from our method perform
significantly better in terms of their ``full'' objective value and their
deviation from a MLE computed using all of the measurement edges.

Beyond providing high-quality sparse measurement graphs, our approach is also
fast. For the \Intel{} dataset, all solutions were obtained in less than 250
milliseconds. Sparsifying the (larger) \AIS2Klinik{} dataset required up to 700
ms, but only around 100 ms when larger edge selections were allowed, as the
duality gap tolerance was reached in fewer than the maximum allowed iterations
of Frank-Wolfe method. The largest dataset (in terms of candidate edges) is the
\City10K{} dataset, with over 10000 loop closure measurements to select from.
Despite this, our approach produces near-optimal solutions in just 2 seconds.

With respect to the suboptimality guarantees of our approach, it is interesting
to note that on both the \Intel{} and \City10K{} datasets, the rounding
procedure introduces fairly significant degradation in algebraic connectivity -
particularly for more aggressive sparsity constraints. In these cases, it seems
that the Boolean relaxation we consider leads to fractional optimal solutions,
rather than solutions amounting to hard selection of just a few edges. It is not
clear in these cases whether the integral solutions obtained by rounding are
indeed suboptimal for the Problem \ref{prob:max-aug-alg-conn}, or whether this
is a consequence of the \emph{integrality gap} between \emph{global} optima of
the relaxation and of Problem \ref{prob:max-aug-alg-conn}.\footnote{In general,
  even simply \emph{verifying} the global optimality of solutions to Problem
  \ref{prob:max-aug-alg-conn} is NP-Hard \cite{mosk2008maximum}.}

At present we do not have access to an implementation of the D-optimal
sparsification approaches considered in \cite{khosoussi2019reliable}. However,
we evaluate our method on similar (and similarly sized) datasets, and a
comparison of the computation times suggests that our approach compares
favorably in computation time and (consequently) the scale of problems we can
consider. For example, \citet[Sec. 9.4]{khosoussi2019reliable} report
computation times of ``$\gg 10$ minutes'' to solve a convex relaxation of the
D-optimal sparsification problem on the \City10K{} dataset (versus $\approx 2$
seconds per Figure \ref{fig:quantitative:city10k}). In light of this fact, and
since both the D-optimality criterion and E-optimality criterion are essentially
variance-minimizing criteria, in the event that one requires D-optimal designs
specifically, an interesting avenue for future work would be to use our
E-optimal designs to supply an initial estimate to, for example, the convex
relaxation approach of \citet{khosoussi2019reliable}. A detailed empirical
comparison of the impact of different optimal design criteria on the quality of
SLAM solutions would be tremendously helpful for practitioners, and would also
be an interesting area for future work.

\section{Conclusion and Future Work}

In this paper, we proposed an approach for pose-graph measurement sparsification
by maximizing the \emph{algebraic connectivity} of the measurement graphs, a key
quantity which has been shown to control the estimation error of pose-graph SLAM
solutions. Our algorithm, \MAC{}, is based on a first-order optimization
approach for solving a convex relaxation of the maximum algebraic connectivity
augmentation problem. The algorithm itself is simple and computationally
inexpensive, and, as we showed, admits formal \emph{post hoc} performance
guarantees on the quality of the solutions it provides. In experiments on
several benchmark pose-graph SLAM datasets, our approach quickly produces
high-quality sparsification results which better preserve the connectivity of
the graph and, consequently, the quality of SLAM solutions computed using those
graphs. An interesting area for future work is the empirical comparison of
different optimality criteria for the pose graph sparsification problem.
Finally, in this work we consider only the removal of measurement graph
\emph{edges}. For lifelong SLAM applications, an important aspect of future work
will be to combine these procedures with methods for \emph{node} removal (e.g.
\cite{Johannsson12rssw, CarlevarisBianco13iros, Carlone2014eliminating}).

\appendices
\section{}

\subsection{Subgradients of the Fiedler value}\label{app:gradients}

In this appendix we consider the problem of computing a supergradient of
$\objectiveF(\selection) = \lambda_2(\LapRotW(\selection))$ with respect to
$\selection$. Strictly speaking, $\objectiveF$ need not be differentaible at a
particular $\selection$ (which occurs specifically when
$\lambda_2(\LapRotW(\selection))$ appears with multiplicity greater than 1, i.e.
it is not a \emph{simple} eigenvalue). We say that a vector $g \in \R^m$ is a
\emph{supergradient} of a concave function $\objectiveF$ at $\selection$ if, for
all $y, x$ in the domain of $\objectiveF$:
\begin{equation}\label{eq:superdifferential-def}
  \objectiveF(y) - \objectiveF(x) \leq g\transpose(y - x).
\end{equation}
Equation \eqref{eq:superdifferential-def} generalizes the notion of
differentiability to the scenario where the function $\objectiveF$ may not be
(uniquely) differentiable a particular point.
We call the set of all supergradients at a particular value of $\selection$ the
\emph{superdifferential} of $\objectiveF$ at $\selection$, denoted $\partial
\objectiveF(\selection)$ \cite{rockafellar2015convex}.

We aim to prove the statement that $\nabla \objectiveF(\selection)$ as defined
in equation \eqref{eq:supergradient} is a supergradient of $\objectiveF$ at
$\selection$.

\begin{proof}[Proof of Theorem \ref{thm:supergradient}]
  We aim to prove the claim by way of equation \eqref{eq:superdifferential-def}.
  Let $u, v \in \R^m,\ \|u\|_2 = \|v\|_2 = 1$ be \emph{any} normalized
  eigenvectors of $\LapRotW(x)$ and $\LapRotW(y)$ with corresponding eigenvalues
  $\lambda_2(\LapRotW(x))$ and $\lambda_2(\LapRotW(y))$, respectively. By
  definition, then, $u$ and $v$ are Fiedler vectors of $\LapRotW(x)$ and
  $\LapRotW(y)$, respectively. Then the left-hand side of equation
  \eqref{eq:superdifferential-def} can be written as:
  \begin{equation}\label{eq:func-diff}
    \begin{aligned}
      \objectiveF(y) - \objectiveF(x) &= \lambda_2(\LapRotW(y)) - \lambda_2(\LapRotW(x)) \\
      &= v\transpose \LapRotW(y) v - u\transpose \LapRotW(x) u.
    \end{aligned}
  \end{equation}
  Now, substitution of $u$ for the Fiedler vector into the definition in
  \eqref{eq:supergradient}, reveals that the $k$-th element of $\nabla
  \objectiveF(x)$ is:
  \begin{equation}
    \nabla \objectiveF(x)_k = u\transpose \LapRotWC_k u.
  \end{equation}
  In turn, the right-hand side of \eqref{eq:superdifferential-def} can be
  written as:
  \begin{equation}\label{eq:differential-linear-approx}
    \nabla \objectiveF(x)\transpose (y - x) = \sum_{k=1}^m  (y_k - x_k) u\transpose \LapRotWC_k u.
  \end{equation}

  Since $u$ and $v$ are minimizers of their respective Rayleigh quotient
  minimization problems, we know:
  \begin{equation}\label{eq:laprotwy-optimality}
    \begin{aligned}
      v\transpose \LapRotW(y) v &\leq u\transpose \LapRotW(y) u \\
      &= u\transpose \LapRotWO u + \sum_{k=1}^m y_k u\transpose \LapRotWC_k u,
    \end{aligned}
  \end{equation}
  where the first line follows from the optimality of $v$ with respect to the
  Rayleigh quotient for $\LapRotW(y)$ and the second line follows from the
  definition of $\LapRotW(y)$. Consider ``adding zero'' to each $y_k$ in
  \eqref{eq:laprotwy-optimality} as $x_k - x_k$ to obtain an equivalent
  expression:
  \begin{equation}\label{eq:laprotw-to-gradient-deriv}
    \begin{aligned}
      \sum_{k=1}^m y_k u\transpose \LapRotWC_k u  &=  \sum_{k=1}^m (y_k + x_k - x_k) u\transpose \LapRotWC_k u, \\
      &= \sum_{k=1}^m x_k u\transpose \LapRotWC_k u + \sum_{k=1}^m (y_k - x_k) u\transpose \LapRotWC_k u.
    \end{aligned}
  \end{equation}
  Comparison to \eqref{eq:differential-linear-approx} reveals that the last term
  in \eqref{eq:laprotw-to-gradient-deriv} is \emph{exactly} equal to $\nabla
  \objectiveF(x)\transpose (y -x)$. In turn, substitution of
  \eqref{eq:differential-linear-approx} into
  \eqref{eq:laprotw-to-gradient-deriv} gives:
  \begin{equation}
    \sum_{k=1}^m y_k u\transpose \LapRotWC_k u = \sum_{k=1}^m x_k u\transpose \LapRotWC_k u + \nabla \objectiveF(x)\transpose (y -x).
  \end{equation}
  Substitution back into \eqref{eq:laprotwy-optimality} gives the bound:
  \begin{equation}
    v\transpose \LapRotW(y) v \leq u\transpose \LapRotWO u + \sum_{k=1}^m x_k u\transpose \LapRotWC_k u + \nabla \objectiveF(x)\transpose (y - x).
  \end{equation}
  Finally, from the definition of $\LapRotW(x)$, we obtain
  \begin{equation}
    v\transpose \LapRotW(y) v \leq u\transpose \LapRotW(x) u +  \nabla \objectiveF(x)\transpose (y - x).
  \end{equation}
  Subtracting $u\transpose \LapRotW(x) u$ from both sides and substituting into
  \eqref{eq:func-diff} gives the desired result.
\end{proof}

\subsection{Solving the direction-finding subproblem}\label{app:dir-subproblem}

This appendix aims to prove the claim that \eqref{eq:dir-subproblem-opt}
provides an optimal solution to the linear program in Problem
\ref{prob:dir-subproblem}.

\begin{proof}[Proof of Theorem \ref{thm:dir-subproblem}]
Rewriting the objective from Problem \ref{prob:dir-subproblem} in terms
of the elements of $s$ and $\nabla \objectiveF(\selection)$, we have:
\begin{equation}\label{eq:dir-subproblem-objective-deriv}
  \begin{aligned}
    s\transpose \nabla \objectiveF(\selection) &= \sum_{k=1}^m s_k \nabla \objectiveF(\selection)_k \\
    &= \sum_{k=1}^m s_k y^*(\selection)\transpose \LapRotWC_k y^*(\selection),
  \end{aligned}
\end{equation}
where in the last line we have used the definition of $\nabla
\objectiveF(\selection)_k$ in \eqref{eq:supergradient}. Now, since each
$\LapRotWC_k \succeq 0$, every component of the gradient must always be
nonnegative, i.e. $\nabla \objectiveF(\selection)_k \geq 0$. Further, since $0
\leq s_k \leq 1$, the objective in \eqref{eq:dir-subproblem-objective-deriv} is
itself a sum of nonnegative terms. From this, it follows directly that the
objective in \eqref{eq:dir-subproblem-objective-deriv} is maximized (subject to
the constraint that $\sum_{k=1}^m s_k = K$) specifically by selecting (i.e. by
setting $s_k = 1$) each of the $K$ largest components of $\nabla
\objectiveF(\selection)$, giving the result in \eqref{eq:dir-subproblem-opt}.
\end{proof}

\bibliographystyle{IEEEtranN}
\begin{footnotesize}
\bibliography{references, thesis-main}
\end{footnotesize}

\end{document}